\documentclass[twoside]{article}

\usepackage{ecj,palatino,epsfig,latexsym}

\usepackage{url}
\usepackage{bm}

\usepackage[utf8]{inputenc}
\usepackage[T1]{fontenc}
\usepackage{lmodern}

\usepackage[american]{babel}
\usepackage{csquotes}

\usepackage[babel = true]{microtype}
\usepackage{xspace}

\usepackage[dvipsnames]{xcolor}
\definecolor{stroke1}{HTML}{2574A9}

\usepackage{amsmath}
\usepackage{amssymb}
\usepackage{amsthm}
\usepackage{dsfont}
\usepackage{mathtools}
\usepackage{enumitem}

\newtheorem{theorem}{Theorem}
\newtheorem{lemma}{Lemma}
\newtheorem{corollary}{Corollary}
\newtheorem{example}{Example}

\usepackage
[
	ruled,
	vlined,
	linesnumbered,
]
{algorithm2e}

\usepackage
{natbib}

\usepackage
[
	bookmarks = true,
	bookmarksopen = false,
	bookmarksnumbered = true,
	pdfstartpage = 1,
	colorlinks = true,
	allcolors = stroke1,
]
{hyperref}
\usepackage
[
	noabbrev,
	nameinlink,
]
{cleveref}

\newcommand*{\N}{\mathds{N}}
\newcommand*{\R}{\mathds{R}}
\newcommand*{\OM}{\textsc{OneMax}\xspace}
\newcommand*{\Jump}{\textsc{Jump}\xspace}
\newcommand*{\LO}{\textsc{LeadingOnes}\xspace}
\newcommand*{\DLB}{\textsc{DeceptiveLeadingBlocks}\xspace}
\newcommand*{\dlb}{\textsc{DLB}\xspace}
\newcommand*{\DB}{\textsc{DeceptiveBlock}\xspace}
\newcommand*{\db}{\textsc{DB}\xspace}
\newcommand*{\PREFIX}{\textsc{Prefix}\xspace}
\newcommand*{\prefix}{\textsc{Prefix}\xspace}
\newcommand*{\E}{\mathrm{E}}

\newcommand{\oea}{\mbox{$(1 + 1)$~EA}\xspace}

\begin{document}

\ecjHeader{x}{x}{xxx-xxx}{2020}{UMDA Copes Well With Deception and Epistasis}{B. Doerr and M. S. Krejca}
\title{\bf The Univariate Marginal Distribution Algorithm Copes Well With Deception and Epistasis}  

\author{\name{\bf Benjamin Doerr} \hfill \addr{doerr@lix.polytechnique.fr}\\ 
    \addr{Laboratoire d'Informatique (LIX), CNRS, \'Ecole Polytechnique, Institut Polytechnique de Paris, Palaiseau, France}
    \AND
    \name{\bf Martin~S. Krejca} \hfill \addr{martin.krejca@hpi.de}\\
    \addr{Hasso Plattner Institute, University of Potsdam, Potsdam, Germany}
}

\maketitle

\begin{abstract}
    In their recent work, Lehre and Nguyen (FOGA 2019)
    show that the univariate marginal distribution algorithm (UMDA) needs time exponential in the parent populations size to optimize the \DLB (\dlb) problem. They conclude from this result that univariate EDAs have difficulties with deception and epistasis.
    
    In this work, we show that this negative finding is caused by an unfortunate choice of the parameters of the UMDA. When the population sizes are chosen large enough to prevent genetic drift, then the UMDA optimizes the DLB problem with high probability with at most $\lambda(\frac{n}{2} + 2 e \ln n)$ fitness evaluations.
    Since an offspring population size~$\lambda$ of order $n \log n$ can prevent genetic drift, the UMDA can solve the \dlb problem with $O(n^2 \log n)$ fitness evaluations. In contrast, for classic evolutionary algorithms no better run time guarantee than $O(n^3)$ is known (which we prove to be tight for the \oea), so our result rather suggests that the UMDA can cope well with deception and epistatis.
    
    From a broader perspective, our result shows that the UMDA can cope better with local optima than evolutionary algorithms; such a result was previously known only for the compact genetic algorithm. Together with the lower bound of Lehre and Nguyen, our result for the first time rigorously proves that running EDAs in the regime with genetic drift can lead to drastic performance losses.
\end{abstract}


\section{Introduction}
\label{sec:introduction}

Estimation-of-distribution algorithms (EDAs) are randomized search heuristics that evolve a probabilistic model of the search space in an iterative manner. Typically starting with the uniform distribution, an EDA takes samples from its current model and then adjusts it such that better solutions have a higher probability of being generated in the next iteration. This method of refinement leads to gradually better solutions and performs well on many practical problems, often outperforming competing approaches~\citep{PelikanHL15SurveyOnEDAs}.

Theoretical analyses of EDAs also often suggest an advantage of EDAs when compared to evolutionary algorithms (EAs); for an in-depth survey of run time results for EDAs, please refer to the article by \cite{KrejcaW18EDAoverview}. With respect to simple unimodal functions, EDAs seem to be comparable to EAs. For example, \cite{SudholtW19cGAandACOonOneMax} proved that the two EDAs cGA and $2$-MMAS$_\textrm{ib}$ have an expected run time of $\Theta(n \log n)$ on the standard theory benchmark function \OM (assuming optimal parameter settings; $n$ being the problem size), which is a run time that many EAs share. The same is true for the EDA UMDA, as shown by the results of \cite{KrejcaW20UMDAlowerBoundOneMax}, \cite{LehreN17UMDAonOneMax}, and \cite{Witt19}. For the benchmark function \LO, \cite{DangL15UMDAonLO} proved an expected run time of $O(n^2)$ for the UMDA when setting the parameters right, which is, again, a common run time bound for EAs on this function. One result suggesting that EDAs can outperform EAs on unimodal functions was given by \cite{DoerrK18sigcGA}. They proposed an EDA called sig-cGA, which has an expected run time of $O(n \log n)$ on both \OM and \LO~-- a performance not known for any classic EA or EDA.

For the class of all linear functions, EDAs perform slightly worse than EAs. The classical $(1 + 1)$ evolutionary algorithm optimizes all linear functions in time $O(n \log n)$ in expectation~\citep{DrosteJW02OnePlusOneEAonJump}. In contrast, the conjecture of \cite{Droste06cGAonLinearFunctions} that the cGA does not perform equally well on all linear functions was recently proven by \cite{Witt18DominoConvergence}, who showed that the cGA has an $\Omega(n^2)$ expected run time on the binary-value function. We note that the binary-value function was found harder also for classical EAs. While the $(1 + \lambda)$ evolutionary algorithm optimizes \OM with $\Theta(n \lambda \log \log \lambda / \log \lambda)$ fitness evaluations, it takes $\Theta(n \lambda)$ fitness evaluations for the binary-value functions~\citep{DoerrK15LambdaEAonLinearFunctions}.

For the mulimodal $\Jump_k$ benchmark function, which has local optima with a Hamming distance of~$k$ away from the global optimum, EDAs seem to drastically outperform EAs. \cite{HasenoehrlS18cGAonJump} recently proved that the cGA only has a run time of $\exp\!\big(O(k + \log n)\big)$ with high probability. \cite{Doerr19cGAonJump} proved that the cGA with high probability has a run time of $O(n \log n)$ on $\Jump_k$ if $k < \tfrac{1}{20} \ln n$, meaning that the cGA is unfazed by the gap of~$k$ separating the local from the global optimum. In contrast, common EAs have run times such as $\Theta(n^k)$~(\oea, \citealp{DrosteJW02OnePlusOneEAonJump}), $\Omega(n^k)$~(${(\mu,\lambda)}$~EA, \citealp{Doerr20gecco}),  and $\Theta(n^{k - 1})$~(${(\mu+1)}$~GA, \citealp{DangFKKLOSS18EAonJumpWithCrossover}), and only go down to smaller run times such as $O(n\log n + kn + 4^k)$ by using crossover in combination with diversity mechanisms like island models~\citep{DangFKKLOSS16EAonJumpDiversityMechanisms}. 

Another result in favor of EDAs was given by \cite{ChenLTY09UMDAonSubstring}, who introduced the \textsc{SubString} function and proved that the UMDA optimizes it in polynomial time, whereas the $(1 + 1)$ evolutionary algorithm has an exponential run time, both with high probability. In the \textsc{SubString} function, only substrings of length $\alpha n$, for $\alpha \in (0, 1)$, of the global optimum are relevant to the fitness of a solution, and these substrings provide a gradient to the optimum. In the process, the $(1 + 1)$ evolutionary algorithm loses bits that are not relevant anymore for following the gradient (but relevant for the optimum). The UMDA fixes its model for correct positions while it is following the gradient and thus does not lose these bits.

The first, and so far only, result to suggest that EDAs can be drastically worse than EAs was recently stated by \cite{LehreN19UMDAonDLB} via the \DLB function (\dlb for short), which they introduce and which consists of blocks of size~$2$ that need to be solved sequentially. Each block is deceptive in the following sense: the values $10$ and $01$ have the worst fitness among all four possible values of a block. The value $00$ has the second best fitness, and the value $11$ has the best fitness. That is, the Hamming distance to the optimal block value~$11$ is not monotone with respect to increasing fitness, and the value~$00$ represents a local optimum for a block. This is why \dlb is considered deceptive.

\cite{LehreN19UMDAonDLB} prove that many common EAs optimize \dlb within $O(n^3)$ fitness evaluations in expectation, whereas the UMDA has a run time of $e^{\Omega(\mu)}$ (where~$\mu$ is an algorithm-specific parameter that often is chosen as a small power of $n$) for a large regime of parameters. Only for extreme parameter values $\lambda = \Omega(\mu^2)$ (where~$\lambda$ is another algorithm-specific parameter that is often chosen in a similar order of magnitude to~$\mu$), the authors prove an expected run time of $O(n^3)$ also for the UMDA.

In this paper, we prove that the UMDA is, in fact, able to optimize \dlb in time $O(n^2 \log n)$ with high probability if its parameters are chosen more carefully (\Cref{thm:UMDA_on_LO}). Note that our result is better than any of the run times proven in the paper by \cite{LehreN19UMDAonDLB}. We achieve this run time by choosing the parameters of the UMDA such that its model is unlikely to degenerate during the run time (\Cref{lem:frequencies_do_not_drop_too_low}). Here by \emph{degenerate} we mean that the sampling frequencies approach the boundary values~$0$ and~$1$ where this is not justified by the objective function. This leads to a probabilistic model that is strongly concentrated around a single search point. This effect is often called \emph{genetic drift}~\citep{SudholtW19cGAandACOonOneMax}. While it appears natural to choose the parameters of an EDA as to prevent genetic drift, it also has been proven that genetic drift can lead to a complicated run time landscape and inferior performance (see \cite{LenglerSW18cGAmediumStepSizes} for the cGA).

In contrast to our setting, for their exponential lower bound, \cite{LehreN19UMDAonDLB} use parameters that lead to genetic drift. Once the probabilistic model is sufficiently degenerated, that is, the frequencies of a block in the first half are $O(1/n)$, the progress of the UMDA is so slow that even to leave the local optima of \dlb (that have a better search point in Hamming distance two only), the EDA takes time exponential in $\mu$.

Since the UMDA shows a good performance in the (more natural) regime without genetic drift and was shown inferior only in the regime with genetic drift, we disagree with the statement of \cite{LehreN19UMDAonDLB} that there are ``inherent limitations of univariate EDAs against deception and epistasis''.

In addition to the improved run time, we derive our result using only tools commonly used in the analysis of EDAs and EAs, whereas the proof of the polynomial run time of $O(n^3)$ for the UMDA with uncommon parameter settings~\citep{LehreN19UMDAonDLB} uses the level-based population method~\citep{Lehre11,DangL16algo,CorusDEL18,DoerrK19}, which is an advanced tool that can be hard to use. We are thus optimistic that our analysis method can be useful also in other run time analyses of EDAs.

We recall that the previous work~\citep{LehreN19UMDAonDLB} only proved upper bounds for the run time of EAs on \dlb, namely of order $O(n^3)$ unless overly large population sizes are used. To support our claim that the UMDA shows a better performance on \dlb than EAs, we rigorously prove a lower bound of $\Omega(n^3)$ for the run time of the \oea on \dlb (\Cref{thm:oneOneEAonDLB}). More precisely, we determine a precise expression for the expected run time of this algorithm on \dlb, which is asymptotically equal to $(1 \pm o(1)) \frac{e-1}{4} n^3$. In addition, we prove that a run time of $\Omega(n^3)$ holds with overwhelming probability (\Cref{thm:eaWHPBound}).

Last, we complement our theoretical result with an empirical comparison of the UMDA to various other evolutionary algorithms. The outcome of these experiments suggests that the UMDA outperforms the competing approaches while also having a smaller variance (\Cref{fig:experiments}). Further, we compare the UMDA to the EDA MIMIC~\citep{BonetIV96MIMIC}, which is similar to the UMDA but uses a more sophisticated probabilistic model that is capable of capturing dependencies among bit positions. Our comparison shows that, for the same parameter regime of both algorithms, the UMDA and the MIMIC behave asymptotically equally, with the MIMIC being slightly faster. This, again, highlights that the UMDA is well suited to optimize \dlb.

The remainder of this paper is structured as follows: in \Cref{sec:preliminaries}, we introduce our notation, formally define \dlb and the UMDA, and we state the tools we use in our analysis. \Cref{sec:runTimeResult} contains our main result (\Cref{thm:UMDA_on_LO}) and discusses its proof informally before stating the different lemmas used to prove it. In \Cref{sec:lb}, we conduct a tight run time analysis of the \oea on \dlb. In \Cref{sec:experiments}, we discuss our empirical results. Last, we conclude this paper in \Cref{sec:conclusion}.

This paper extends our conference version~\citep{DoerrK20UMDAonDLB} in two major ways: (i) we prove lower bounds of $\Omega(n^3)$ for the run time of the \oea on \dlb, which hold in expectation (\Cref{thm:oneOneEAonDLB}) and with overwhelming probability (\Cref{thm:eaWHPBound}). (ii) Our empirical analysis contains more EAs, the MIMIC (\Cref{fig:experiments}), and also considers the impact of the UMDA's population size (\Cref{fig:variableMu}).

\section{Preliminaries}
\label{sec:preliminaries}

We are concerned with the run time analysis of algorithms optimizing pseudo-Boolean functions, that is, functions $f\colon \{0, 1\}^n \to \R$, where $n \in \N^+$ denotes the dimension of the problem. Given a pseudo-Boolean function~$f$ and a bit string~$x$, we refer to~$f$ as a \emph{fitness function,} to~$x$ as an \emph{individual,} and to $f(x)$ as the \emph{fitness of~$x$.}

For $n_1, n_2 \in \N \coloneqq \{0, 1, 2, \ldots\}$, we define $[n_1 .. n_2] = [n_1, n_2] \cap \N$, and for an $n \in \N$, we define $[n] = [1 .. n]$. From now on, if not stated otherwise, the variable~$n$ always denotes the problem size. For a vector~$x$ of length~$n$, we denote its component at index $i \in [n]$ by $x_i$ and, for and index set $I \subseteq [n]$, we denote the subvector of length~$|I|$ consisting only of the components at indices in~$I$ by $x_I$. Further, let $|x|_1$ denote the number of~$1$s of~$x$ and $|x|_0$ its number of~$0$s.

\subsubsection*{\DLB.}

The pseudo-Boolean function \DLB (abbreviated as \dlb) was introduced by \cite{LehreN19UMDAonDLB} as a deceptive version of the well known benchmark function \LO. In \dlb, an individual~$x$ of length~$n$ is divided into blocks of equal size~$2$. Each block consists of a trap, where the fitness of each block is determined by the number of~$0$s (minus $1$), except that a block of all $1$s has the best fitness of~$2$. The overall fitness of~$x$ is then determined by the longest prefix of blocks with fitness~$2$ plus the fitness of the following block.
Note that in order for the chunking of \dlb to make sense, it needs to hold that~$2$ divides~$n$. In the following, we always assume this implicitly.

We now provide a formal definition of \dlb. To this end, we first introduce the function $\DB\colon \{0, 1\}^2 \to [0 .. 2]$ (abbreviated as \db), which determines the fitness of a block (of size~$2$). For all $x \in \{0, 1\}^2$, we have
\begin{align*}
    \db(x) =
    \begin{cases}
        2 & \textrm{if } |x|_1 = 2,\\
        |x|_0 - 1 & \textrm{else.}
    \end{cases}
\end{align*}
Further, we define the function $\PREFIX\colon \{0, 1\}^n \to [0 .. n]$, which determines the longest prefix of~$x$ with blocks of fitness~$2$. For a logic formula~$P$, let $[P]$ denote the Iverson bracket, that is,~$[P]$ is~$1$ if~$P$ is true and~$0$ otherwise. We define, for all $x \in \{0, 1\}^n$,
\begin{align*}
    \PREFIX(x) = \sum_{i = 1}^{n/2} \big[\forall j \leq i\colon \db(x_{\{2i - 1, 2i\}}) = 2\big].
\end{align*}

\dlb is now defined as follows for all $x \in \{0, 1\}^n$:
\begin{align*}
    \dlb(x) =
    \begin{cases}
        n & \textrm{if } \PREFIX(x) = n,\\
        \sum_{i = 1}^{\PREFIX(x) + 1} \db(x_{\{2i - 1, 2i\}}) & \textrm{else.}
    \end{cases}
\end{align*}

\subsubsection*{The univariate marginal distribution algorithm.}

Our algorithm of interest is the UMDA (\citealp{MuehlenbeinP96UMDA}; \Cref{alg:UMDA}) with parameters $\mu, \lambda \in \N^+$, $\mu \leq \lambda$. It maintains a vector~$p$ \emph{(frequency vector)} of probabilities \emph{(frequencies)} of length~$n$ as its probabilistic model. This vector is used to sample an individual $x \in \{0, 1\}^n$, which we denote as $x \sim \textrm{sample}(p)$, such that, for all $y \in \{0, 1\}^n$,
\[
\Pr[x = y] = \prod_{\substack{i = 1:\\y_i = 1}}^{n} p_i \prod_{\substack{i = 1:\\y_i = 0}}^{n} (1 - p_i).
\]
The UMDA updates this vector iteratively in the following way: first, $\lambda$ individuals are sampled. Then, among these~$\lambda$ individuals, a subset of~$\mu$ with the highest fitness is chosen (breaking ties uniformly at random), and, for each index $i \in [n]$, the frequency~$p_i$ is set to the relative number of~$1$s at position~$i$ among the~$\mu$ best individuals. Last, if a frequency~$p_i$ is below~$\frac{1}{n}$, it is increased to~$\frac{1}{n}$, and, analogously, frequencies above $1 - \frac{1}{n}$ are set to $1 - \frac{1}{n}$. Capping into the interval $[\frac{1}{n}, 1 - \frac{1}{n}]$ circumvents frequencies from being stuck at the extremal values~$0$ or~$1$. Last, we denote the frequency vector of iteration $t \in \N$ with $p^{(t)}$.

\begin{algorithm}
    \caption{\label{alg:UMDA} The UMDA~\citep{MuehlenbeinP96UMDA} with parameters~$\mu$ and~$\lambda$, $\mu \leq \lambda$, maximizing a fitness function $f\colon \{0, 1\}^n \to \R$ with $n \geq 2$}
    
    $t \gets 0$\;
    $p^{(t)} \gets (\tfrac{1}{2})_{i \in [n]}$ \tcp*{vector of length~$n$ with all entries being~$1/2$}
    \Repeat( \texttt{// iteration}~$t$)
    {\emph{termination criterion met}}
    {
        \lFor{$i \in [\lambda]$}
        {
            $x^{(i)} \sim \mathrm{sample}\!\left(p^{(t)}\right)$%
        }
        let $y^{(1)}, \ldots, y^{(\mu)}$ denote the~$\mu$ best individuals out of $x^{(1)}, \ldots, x^{(\lambda)}$ (breaking ties uniformly at random)\;
        \lFor{$i \in [n]$}
        {
            $p^{(t + 1)}_i \gets \frac{1}{\mu} \sum_{j = 1}^{\mu} y^{(j)}_i$%
        }
        restrict~$p^{(t + 1)}$ to the interval $[\tfrac{1}{n}, 1 - \tfrac{1}{n}]$\;
        $t \gets t + 1$\;
    }
\end{algorithm}

\subsubsection*{Run time analysis.}

When analyzing the run time of the UMDA optimizing a fitness function~$f$, we are interested in the number~$T$ of fitness function evaluations until an optimum of~$f$ is sampled for the first time. Since the UMDA is a randomized algorithm, this run time $T$ is a random variable. Note that the run time of the UMDA is at most~$\lambda$ times the number of iterations until an optimum is sampled for the first time, and it is at least~$\lambda$ times this number minus $(\lambda - 1)$.

In the area of run time analysis of randomized search heuristics, it is common to give bounds for the expected value of the run time of the algorithm under investigation. This is uncritical when the run time is concentrated around its expectation, as often observed for classical evolutionary algorithms. For EDAs, it has been argued, among others by~\cite{Doerr19cGAonJump}, that it is preferable to give bounds that hold with high probability. This is what we shall aim at in this work as well. Of course, it would be even better to give estimates in a distributional sense, e.g., via stochastic domination by another distribution, as argued for by~\cite{Doerr19tcs}, but this appears to be difficult for EDAs, among others, because of the very different behavior in the regimes with and without strong genetic drift. 

\subsubsection*{Probabilistic tools.}

We use the following results in our analysis. In order to prove statements on random variables that hold with high probability, we use the following commonly known Chernoff bound.

\begin{theorem}[Chernoff bound;~{\citealp[Theorem~$10.5$]{Doerr20bookchapter}, \citealp{Hoeffding63ChernoffBound}}]
    \label{thm:chernoff}
    Let $k \in \N$, $\delta \in [0, 1]$, and let~$X$ be the sum of~$k$ independent random variables, each taking values in $[0, 1]$. Then
    \begin{align*}
        \Pr\!\big[X \leq (1 - \delta) \E[X]\big] \leq \exp\!\left(- \frac{\delta^2 \E[X]}{2}\right).
    \end{align*}
\end{theorem}

The next lemma tells us that, for a random~$X$ following a binomial law, the probability of exceeding $\E[X]$ is bounded from above by roughly the term with the highest probability.

\begin{lemma}[{\citealp[Eq.~(10.62)]{Doerr20bookchapter}}]
    \label{lem:binExceedingExpectedValue}
    Let $k \in \N$, $p \in [0,1]$, $X \sim \mathrm{Bin}(k, p)$, and let $m \in \big[\E[X] + 1 .. k\big]$. Then
    \begin{align*}
        \Pr[X \geq m] \leq \frac{m(1 - p)}{m - \E[X]} \cdot \Pr[X = m].
    \end{align*}
\end{lemma}

We use \Cref{lem:binExceedingExpectedValue} for the converse case, that is, in order to bound the probability that a binomially distributed random variable is smaller than its expected value.

\begin{corollary}
    \label{cor:binBelowExpectedValue}
    Let $k \in \N$, $p \in [0,1]$, $X \sim \mathrm{Bin}(k, p)$, and let $m \in \big[0 .. \E[X] - 1\big]$. Then
    \begin{align*}
        \Pr[X \leq m] \leq \frac{(k - m)p}{\E[X] - m} \cdot \Pr[X = m].
    \end{align*}
\end{corollary}

\begin{proof}
    Let $\overline{X} \coloneqq k - X$, and let $\overline{m} \coloneqq k - m$. Note that $\overline{X} \sim \mathrm{Bin}(k, 1 - p)$ with $\E[\overline{X}] = k - \E[X]$ and that $\overline{m} \in \big[\E[\overline{X}] + 1 .. k\big]$. With \Cref{lem:binExceedingExpectedValue}, we compute
    \begin{align*}
        \Pr[X \leq m] = \Pr[\overline{X} \geq \overline{m}] \leq \frac{\overline{m}p}{\overline{m} - \E[\overline{X}]} \cdot \Pr[\overline{X} = \overline{m}] = \frac{(k - m)p}{\E[X] - m} \cdot \Pr[X = m],
    \end{align*}
    which proves the claim.
\end{proof}

Last, the following theorem deals with a \emph{neutral} bit in a fitness function~$f$, that is, a position $i \in [n]$ such that bit values at~$i$ do not contribute to the fitness value at all. The following theorem by~\cite{DoerrZ20tec} states that if the UMDA optimizes such an~$f$, then the frequency at position~$i$ stays close to its initial value~$\frac{1}{2}$ for $\Omega(\mu)$ iterations. We go more into detail about how this relates to \dlb at the beginning of \Cref{sec:runTimeResult}.

\begin{theorem}[{\citealp[Theorem~$2$]{DoerrZ20tec}}]
    \label{thm:neutralFrequency}
    Consider the UMDA optimizing a fitness function~$f$ with a neutral bit $i \in [n]$. Then, for all $d > 0$ and all $t \in \N$, we have
    \begin{align*}
        \Pr\!\big[\forall t' \in [0 .. t]\colon |p_i^{(t')} - \tfrac{1}{2}| < d\big] \geq 1 - 2 \exp\left(- \frac{d^2 \mu}{2 t}\right).
    \end{align*}
\end{theorem}

\section{Run Time Analysis of the UMDA}
\label{sec:runTimeResult}

In the following, we prove that the UMDA optimizes \dlb efficiently, which is the following theorem.

\newcommand*{\cmu}{c_{\mu}}
\newcommand*{\clambda}{c_{\lambda}}
\begin{theorem}
    \label{thm:UMDA_on_LO}
    Let $\delta, \varepsilon, \zeta \in (0, 1)$ be constants, and let $\cmu = 16/\varepsilon^2$ and $\clambda = (1 - \zeta)(1 - \delta^2)(1 - \varepsilon)^4/(16e)$. Consider the UMDA optimizing \dlb with $\mu \geq \cmu n \ln n$ and $\mu / \lambda \leq \clambda$. Then the UMDA samples the optimum after $\lambda (\frac{n}{2} + 2 e \ln n)$ fitness function evaluations with a probability of at least $1 - 9 n^{-1}$.
\end{theorem}

Before we present the proof, we sketch its main ideas and introduce important notation. We show that the frequencies of the UMDA are set to $1 - \frac{1}{n}$ block-wise from left to right with high probability. We formalize this concept by defining that a block $i \in [\frac{n}{2}]$ is \emph{critical} (in iteration~$t$) if and only if $p^{(t)}_{2i - 1} + p^{(t)}_{2i} < 2 - \frac{2}{n}$ and, for each index $j \in [2i - 2]$, the frequency~$p^{(t + 1)}_j$ is at $1 - \frac{1}{n}$. Intuitively, a critical block is the first block whose frequencies are not at their maximum value. We prove that a critical block is optimized within a single iteration with high probability if we assume that its frequencies are not below $(1 - \varepsilon)/2$, for $\varepsilon \in (0, 1)$ being a constant.

In order to assure that the frequencies of each block are at least $(1 - \varepsilon)/2$ until it becomes critical, we show that most of the frequencies right of the critical block are not impacted by the fitness function. We call such frequencies \emph{neutral.} More formally, a frequency~$p_i$ is neutral in iteration~$t$ if and only if the probability to have a~$1$ at position~$i$ in each of the~$\mu$ \emph{selected} individuals equals~$p^{(t)}_i$. Note that since we assume that $\mu = \Omega(n \log n)$, by \Cref{thm:neutralFrequency}, the impact of the genetic drift on neutral frequencies is low with high probability.

We know which frequencies are neutral and which are not by the following key observation: consider a population of~$\lambda$ individuals of the UMDA during iteration~$t$; only the first (leftmost) block that has strictly fewer than~$\mu$ $11$s is relevant for selection, since the fitness of individuals that do not have a~$11$ in this block cannot be changed by bits to the right anymore. We call this block \emph{selection-relevant.} Note that this is a notion that depends on the random offspring population in iteration~$t$, whereas the notion \emph{critical} depends only on~$p^{(t)}$.

The consequences of a selection-relevant block are as follows: if block $i \in [\frac{n}{2}]$ is selection-relevant, then all frequencies in blocks left of~$i$ are set to $1 - \frac{1}{n}$, since there are at least~$\mu$ individuals with $11$s. All blocks right of~$i$ have \emph{no} impact on the selection process: if an individual has no $11$ in block~$i$, its fitness is already fully determined by all of its bits up to block~$i$ by the definition of \dlb. If an individual has a $11$ in block~$i$, it is definitely chosen during selection, since there are fewer than~$\mu$ such individuals and since its fitness is better than that of all of the other individuals that do not have a~$11$ in block~$i$. Thus, its bits at positions in blocks right of~$i$ are irrelevant for selection. Overall, since the bits in blocks right of~$i$ do not matter, the frequencies right of block~$i$ get no signal from the fitness function and are thus neutral (see \Cref{lem:frequencies_do_not_drop_too_low}).

Regarding block~$i$ itself, all of the individuals with $11$s are chosen, since they have the best fitness. Nonetheless, individuals with a $00$, $01$, or $10$ can also be chosen, where an individual with a $00$ in block~$i$ is preferred, as a $00$ has the second best fitness after a~$11$. Since the fitness for a $10$ or $01$ is the same, selecting individuals with such blocks does not impact the number of~$1$s at the positions in block~$i$ in expectation. However, if more $00$s than $11$s are sampled for block~$i$, it can happen that the frequencies of block~$i$ are decreased. Since we assume that $\mu = \Omega(n \log n)$, the frequency is sufficiently high before the update and the frequencies of block~$i$ do not decrease by much with high probability (see \Cref{lem:drop_of_a_selection_relevant_block}). Since, in the next iteration, block~$i$ is the critical block, it is then optimized within a single iteration (see \Cref{lem:increasing_a_frequency}), and we do not need to worry about its frequencies decreasing again.

\subsubsection*{Neutral frequencies.}

We now prove that the frequencies right of the selection-relevant block do not decrease by too much within the first~$n$ iterations.

\begin{lemma}
    \label{lem:frequencies_do_not_drop_too_low}
    Let $\varepsilon \in (0, 1)$ be a constant. Consider the UMDA with $\lambda \geq \mu \geq (16 n / \varepsilon^2) \ln n$ optimizing \dlb. Let $t \leq n$ be the first iteration such that block $i \in [\frac{n}{2}]$ becomes selection-relevant for the first time. Then, with a probability of at least $1 - 2 n^{-1}$, all frequencies at the positions $[2i + 1 .. n]$ are at least~$(1 - \varepsilon)/2$ within the first~$t$ iterations.
\end{lemma}

\begin{proof}
    Let $j \in [2i + 1 .. n]$ denote the index of a frequency right of block~$i$. Note that by the assumption that~$t$ is the first iteration such that block~$i$ becomes selection-relevant it follows that, for all $t' \leq t$, the frequency~$p^{(t')}_i$ is neutral, as we discussed above.
    
    Since~$p^{(t')}_j$ is neutral for all $t' \leq t$, by \Cref{thm:neutralFrequency} with $d = \frac{\varepsilon}{2}$, we see that the probability that~$p_j$ leaves the interval $\big((1 - \varepsilon)/2, (1 + \varepsilon)/2\big)$ within the first $t \leq n$ iterations is at most $2 \exp\big(-\varepsilon^2 \mu / (8 t)\big) \leq 2 \exp\big(-\varepsilon^2 \mu / (8 n)\big) \leq 2 n^{-2}$, where we used our bound on~$\mu$.
    
    Applying a union bound over all $n - 2i \leq n$ neutral frequencies yields that at least one frequency leaves the interval $\big((1 - \varepsilon)/2, (1 + \varepsilon)/2\big)$ within the first~$t$ iterations with a probability of at most $2 n^{-1}$, as desired.
\end{proof}

\subsubsection*{Update of the selection-relevant block.}

As mentioned at the beginning of the section, while frequencies right of the selection-relevant block do not drop below $(1 - \varepsilon)/2$ with high probability (by \Cref{lem:frequencies_do_not_drop_too_low}), the frequencies of the selection-relevant block can drop below $(1 - \varepsilon)/2$, as the following example shows.

\begin{example}
    Consider the UMDA with $\mu = 0.05 \lambda \geq c \ln n$, for a sufficiently large constant~$c$, optimizing \dlb. Consider an iteration~$t$ and assume that block $i = \frac{n}{2} - 1 - o(n)$  is critical. Assume that the frequencies in blocks~$i$ and~$i + 1$ are all at $2/5$. Since the prefix of $2(i - 1)$ leading~$1$s is sampled with probability $(1 - \frac{1}{n})^{2(i - 1)} \geq (1 - \frac{1}{n})^{n - 1} \geq 1/e$, the offspring population in iteration~$t$ has roughly $\big((2/5)^2 / e\big) \lambda \approx 0.058 \lambda > \mu$ individuals with at least $2i$ leading~$1$s in expectation. By \Cref{thm:chernoff}, this also holds with high probability. Thus, the frequencies in block~$i$ are set to $1 - \frac{1}{n}$ with high probability.
    
    The expected number of individuals with at least~$2i + 2$ leading~$1$s is roughly $\big((2/5)^4/e\big) \lambda \approx 0.0095 \lambda$, and the expected number of individuals with $2i$ leading~$1$s followed by a~$00$ is roughly $\big((2/5)^2 \cdot (3/5)^2 / e\big) \lambda \approx 0.02 \lambda$. In total, we expect approximately $0.0295 \lambda < \mu$ individuals with $2i$ leading~$1$s followed by either a~$11$ or a~$00$. Again, by \Cref{thm:chernoff}, these numbers occur with high probability. Note that this implies that block $i + 1$ is selection-relevant with high probability.
    
    Consider block~$i + 1$. For selection, we choose all $0.0295 \lambda$ individuals with $2i$ leading~$1$s followed by either a~$11$ or a~$00$ (which are sampled with high probability). For the remaining $\mu - 0.0295 \lambda = 0.0205 \lambda$ selected individuals with $2i$ leading~$1$s, we expect half of them, that is, $0.01025 \lambda$ individuals to have a~$10$. Thus, with high probability, the frequency $p^{(t + 1)}_{2i + 1}$ is set to roughly $(0.0095 + 0.01025) \lambda / \mu = 0.395$, which is less than $0.4 = p^{(t)}_{2i + 1}$. Thus, this frequency decreased.
\end{example}

The next lemma shows that such frequencies do not drop too low, however.

\begin{lemma}
    \label{lem:drop_of_a_selection_relevant_block}
    Let $\varepsilon, \delta \in (0, 1)$ be constants, and let $c$ be a sufficiently large constant. Consider the UMDA with
    $\lambda \geq \mu \geq c \ln n$
    optimizing \dlb. Further, consider an iteration~$t$ such that block $i \in [2 .. \frac{n}{2}]$ is selection-relevant, and assume that its frequencies~$p^{(t)}_{2i - 1}$ and~$p^{(t)}_{2i}$ are at least $(1 - \varepsilon)/2$ when sampling the population. Then the frequencies~$p^{(t + 1)}_{2i - 1}$ and~$p^{(t + 1)}_{2i}$ are at least $(1 - \delta)(1 - \varepsilon)^2/4$ with a probability of at least $1 - 4n^{-2}$.
\end{lemma}

\begin{proof}
   Let for the moment $i$ be arbitrary. Let~$k$ denote the number of individuals with a prefix of at least $2i - 2$ leading~$1$s and let $X$ be the random variable that counts how many of these have $1$s also in positions $2i-1$ and~$i$. Then~$X$ follows a binomial law with~$k$ trials and with a success probability of $p^{(t)}_{2i - 1} p^{(t)}_{2i} \eqqcolon \widetilde{p} \geq (1 - \varepsilon)^2/4$. Since in the following we condition on block~$i$ being selection-relevant, it follows that $k \geq \mu$ and $X < \mu$. We now bound the probability that at least $(1 - \delta)\widetilde{p}\mu \eqqcolon m$ individuals have~$2i$ leading~$1$s, that is, we bound $\Pr[X \geq m \mid X < \mu]$.
    
    Elementary calculations show that
    \begin{align}
        \notag
        \Pr[X \geq m \mid X < \mu] &= 1 - \Pr[X < m \mid X < \mu] = 1 - \frac{\Pr[X < m, X < \mu]}{\Pr[X < \mu]}\\
        \label{eq:notManyOneOnes}
        &= 1 - \frac{\Pr[X < m]}{\Pr[X < \mu]}.
    \end{align}
    To show a lower bound for \eqref{eq:notManyOneOnes}, consider separately the two cases that $\E[X] < \mu$ and $\E[X] \geq \mu$.
    
    \textbf{Case 1: $\E[X] < \mu$.} We first bound the numerator of the subtrahend in \eqref{eq:notManyOneOnes}. Since $m/(1 - \delta) = \widetilde{p}\mu \leq \widetilde{p}k = \E[X]$, we have $\Pr[X < m] \leq \Pr\big[X < (1 - \delta)\E[X]\big]$. By \Cref{thm:chernoff}, by $\E[X] \geq \widetilde{p}\mu$, and by our assumption that $\mu \geq c \ln n$, choosing~$c$ sufficiently large, we have
    \begin{align*}
        \Pr\big[X < (1 - \delta)\E[X]\big] \leq \exp\!\left(-\frac{\delta^2 \E[X]}{2}\right) \leq \exp\!\left(-\frac{\delta^2 \widetilde{p}\mu}{2}\right) \leq n^{-2}.
    \end{align*}
    
    For bounding the denominator, we note that $\widetilde{p} \leq 1 - \frac{1}{n}$ and use the fact that a binomially distributed random variable with a success probability of at most $1 - \frac{1}{n}$ is below its expectation with a probability of at least~$\frac{1}{4}$ \citep[Lemma~$10.20$~(b)]{Doerr20bookchapter}. This yields
    \begin{align*}
        \Pr[X < \mu] \geq \Pr\big[X < \E[X]\big] \geq \frac{1}{4}.
    \end{align*}
    
    Combining these bounds, we obtain $\Pr[X \geq m \mid X < \mu] \geq 1 - 4n^{-2}$ for this case.
    
    \textbf{Case 2: $\E[X] \geq \mu > m$.} We bound the subtrahend from \eqref{eq:notManyOneOnes} from above. By basic estimations and by \Cref{cor:binBelowExpectedValue}, we see that
    \begin{align}
        \label{eq:probabilityBoundCaseTwo}
        \frac{\Pr[X < m]}{\Pr[X < \mu]} \leq \frac{\Pr[X \leq m - 1]}{\Pr[X = \mu - 1]} \leq \frac{(k - m + 1)\widetilde{p}}{\E[X] - m + 1} \cdot \frac{\Pr[X = m - 1]}{\Pr[X = \mu - 1]}.
    \end{align}
    
    We bound the first factor of \eqref{eq:probabilityBoundCaseTwo} as follows, recalling that $(1 - \delta)\widetilde{p}\mu = m$ and noting that $m \geq 1$ for sufficiently large values of~$n$:
    \begin{align*}
        \frac{(k - m + 1)\widetilde{p}}{\E[X] - m + 1} &\leq \frac{\E[X]}{\E[X] - m} = 1 + \frac{m}{\E[X] - m} \leq 1 + \frac{m}{\mu - m} \leq 1 + \frac{m}{m/\widetilde{p} - m}\\
        &= 1 + \frac{\widetilde{p}}{1 - \widetilde{p}} \leq 1 + n - 1 = n\ ,
    \end{align*}
    where the last inequality uses that $\widetilde{p} \leq (1 - \frac{1}{n})^2 \leq 1 - \frac{1}{n}$.
    
    For the second factor of \eqref{eq:probabilityBoundCaseTwo}, we compute
    \begin{align}
        \notag
        \frac{\Pr[X = m - 1]}{\Pr[X = \mu - 1]} &= \frac{\binom{k}{m - 1} \widetilde{p}^{m - 1} (1 - \widetilde{p})^{k - m + 1}}{\binom{k}{\mu - 1} \widetilde{p}^{\mu - 1} (1 - \widetilde{p})^{k - \mu + 1}}\\
        \label{eq:probabilityQuotient}
        &= \frac{(\mu - 1)! (k - \mu + 1)!}{(m - 1)! (k - m + 1)!} \cdot \left(\frac{1 - \widetilde{p}}{\widetilde{p}}\right)^{\mu - m}.
    \end{align}
    Since $\widetilde{p} \geq (1 - \varepsilon)^2/4$, we see that $(1 - \widetilde{p})/\widetilde{p} \leq 4/(1 - \varepsilon)^2$.
    
    For the first factor of \eqref{eq:probabilityQuotient}, let $p^* \coloneqq (1 - \delta)\widetilde{p}$, thus $\mu p^* = m$. Noting that, for all $a, b \in \R$ with $a < b$, the function $j \mapsto (a + j)(b - j)$ is maximal for $j = (b - a)/2$, we first bound
    \begin{align*}
        \frac{(\mu - 1)!}{(m - 1)!} &= \prod_{j = 0}^{\mu - m - 1} (\mu - 1 - j)
        = \left(\prod_{j = 0}^{\lfloor(\mu - m - 1)/2\rfloor} (\mu - 1 - j)\right) \left(\prod_{j = \lceil(\mu - m)/2\rceil}^{\mu - m - 1} (\mu - 1 - j)\right)\\
        &\leq \left(\prod_{j = 0}^{\lfloor(\mu - m - 1)/2\rfloor} (\mu - 1 - j)\right) \left(\prod_{j = 0}^{\lfloor(\mu - m - 1)/2\rfloor} (m + j)\right)\\
        &= \prod_{j = 0}^{\lfloor(\mu - m - 1)/2\rfloor} \big((\mu - 1 - j)(m + j)\big) \leq \left(\frac{\mu + m}{2}\right)^{\mu - m} \leq \left(\frac{\mu}{2}(1 + p^*)\right)^{\mu - m}.
    \end{align*}
    Substituting this into the first factor of \eqref{eq:probabilityQuotient}, we bound
    \begin{align*}
        &\frac{(\mu - 1)! (k - \mu + 1)!}{(m - 1)! (k - m + 1)!} \leq \left(\frac{\mu}{2}(1 + p^*)\right)^{\mu - m} \cdot \frac{(k - \mu + 1)!}{(k - m + 1)!}\\
        &\quad= \frac{\left(\frac{\mu}{2}(1 + p^*)\right)^{\mu - m}}{\prod_{j = 0}^{\mu - m - 1} (k - m + 1 - j)} = \prod_{j = 0}^{\mu - m - 1} \frac{\mu(1 + p^*)}{2(k - m + 1 - j)}.
    \end{align*}
    By noting that $k \widetilde{p} = \E[X] \geq \mu$, we bound the above estimate further:
    \begin{align*}
        \prod_{j = 0}^{\mu - m - 1} \frac{\mu(1 + p^*)}{2 (k - m + 1 - j)} &\leq \prod_{j = 0}^{\mu - m - 1} \frac{\mu(1 + p^*)}{2(\mu/\widetilde{p} - m + 1 - j)}\\
        &\leq \left(\frac{\mu(1 + p^*)}{2 \mu(1/\widetilde{p} - 1) + 2}\right)^{\mu - m} \leq \left(\frac{\widetilde{p}(1 + p^*)}{2(1 - \widetilde{p})}\right)^{\mu - m}.
    \end{align*}
    Substituting both bounds into \eqref{eq:probabilityQuotient} and recalling that $m = \mu p^*$, we obtain
    \begin{align*}
        \frac{\Pr[X = m - 1]}{\Pr[X = \mu - 1]} \leq \left(\frac{1 + p^*}{2}\right)^{\mu(1 - p^*)} = \exp\!\left(-\mu (1 - p^*) \ln\left(\frac{2}{1 + p^*}\right)\right).
    \end{align*}
    
    Finally, substituting this back into our bound of \eqref{eq:probabilityBoundCaseTwo}, using our assumption that $\mu \geq c \ln n$ and noting that $p^*$ is constant, choosing~$c$ sufficiently large, we obtain
    \begin{align*}
        \frac{\Pr[X < m]}{\Pr[X < \mu]} \leq n \exp\!\left(-\mu (1 - p^*) \ln\left(\frac{2}{1 + p^*}\right)\right) \leq n^{-2}.
    \end{align*}
    
    \textbf{Concluding the proof.} In both cases, we see that the number of $11$s in block~$i$ is at least $m = (1 - \delta)\widetilde{p}\mu \geq \big((1 - \delta)(1 - \varepsilon)^2/4\big) \mu$ with a probability of at least $1 - 4n^{-2}$. Since each $11$ contributes to the new values of $p_{2i - 1}$ and $p_{2i}$, after the update, both frequencies are at least $(1 - \delta)(1 - \varepsilon)^2/4$ with the same probability bound, as we claimed.
\end{proof}

\subsubsection*{Optimizing the critical block.}

Our next lemma considers the critical block $i \in [\frac{n}{2}]$ of an iteration~$t$. It shows that, with high probability, for all $j \in [2i]$, we have that $p^{(t + 1)}_j = 1 - \frac{1}{n}$. Informally, this means that (i) all frequencies left of the critical block remain at $1 - \frac{1}{n}$, and (ii) the frequencies of the critical block are increased to $1 - \frac{1}{n}$.

\begin{lemma}
    \label{lem:increasing_a_frequency}
    Let $\delta, \varepsilon, \zeta \in (0, 1)$ be constants and let $q = (1 - \delta)^2 (1 - \varepsilon)^4 / 16$. Consider the UMDA optimizing \dlb with $\lambda \geq (4/\zeta^2) \ln n$ and $\mu / \lambda \leq (1 - \zeta) q / e$, and consider an iteration~$t$ such that block $i \in [\frac{n}{2}]$ is critical and that $p^{(t)}_{2i - 1}$ and~$p^{(t)}_{2i}$ are at least $\sqrt{q}$. Then, with a probability of at least $1 - n^{-2}$, at least~$\mu$ offspring are generated with at least~$2i$ leading~$1$s. In other words, the selection-relevant block of iteration~$t$ is at a position in $[i + 1 .. \frac{n}{2}]$.
\end{lemma}

\begin{proof}
    Let~$X$ denote the number of individuals that have at least~$2i$ leading $1$s. Since block~$i$ is critical, each frequency at a position $j \in [2i - 2]$ is at $1 - \frac{1}{n}$. Thus, the probability that all of these frequencies sample a~$1$ for a single individual is $(1 - \frac{1}{n})^{2i - 2} \geq (1 - \frac{1}{n})^{n - 1} \geq 1/e$. Further, since the frequencies~$p^{(t)}_{2i - 1}$ and~$p^{(t)}_{2i}$ are at least $\sqrt{q}$, the probability to sample a $11$ at these positions is at least $q$. Hence, we have $\E[X] \geq q \lambda / e$.
    
    We now apply \Cref{thm:chernoff} to show that it is unlikely that fewer than~$\mu$ individuals from the current iteration have fewer than~$2i$ leading~$1$s. Using our bounds on~$\mu$ and~$\lambda$, we compute
    \begin{align*}
        \Pr[X < \mu] \leq \Pr\left[X \leq (1 - \zeta) \frac{q}{e} \lambda\right] \leq \Pr\big[X \leq (1 - \zeta) \E[X]\big] \leq e^{-\frac{\zeta^2 \lambda}{2 }} \leq n^{-2} .
    \end{align*}
    Thus, with a probability of at least $1 - n^{-2}$, at least~$\mu$ individuals have at least~$2i$ leading~$1$s. This concludes the proof.
\end{proof}

\subsubsection*{The run time of the UMDA on \dlb.}

We now prove our main result.

\begin{proof}[Proof of \Cref{thm:UMDA_on_LO}]
    We prove that the UMDA samples the optimum after $\frac{n}{2} + 2 e \ln n$ iterations with a probability of at least $1 - 9 n^{-1}$. Since it samples~$\lambda$ individuals each iteration, the theorem follows.
    
    Due to \Cref{lem:frequencies_do_not_drop_too_low} and $\mu \geq \cmu n \ln n$, within the first~$n$ iterations, with a probability of at least $1 - 2 n^{-1}$, no frequency drops below $(1 - \varepsilon)/2$ while its block has not been selection-relevant yet.
    
    By \Cref{lem:drop_of_a_selection_relevant_block}, since $\mu = \omega(\log n)$, with a probability of at least $1 - 4n^{-2}$, once a block becomes selection-relevant, its frequencies do not drop below $(1 - \delta)(1 - \varepsilon)^2/4$ for the next iteration. By a union bound, this does not fail for~$n$ consecutive times with a probability of at least $1 - 4n^{-1}$. Note that a selection-relevant block becomes critical in the next iteration.
    
    Consider a critical block $i \in [\frac{n}{2}]$. By \Cref{lem:increasing_a_frequency}, since $\lambda = \omega(\log n)$, with a probability of at least $1 - n^{-2}$, all frequencies at positions in $[2i]$ are immediately set to $1 - \frac{1}{n}$ in the next iteration, and the selection-relevant block has an index of at least $i + 1$, thus, moving to the right. Applying a union bound for the first~$n$ iterations of the UMDA and noting that each frequency belongs to a selection-relevant block at most once shows that all frequencies are at $1 - \frac{1}{n}$ after the first~$\frac{n}{2}$ iterations, since each block contains two frequencies, and stay there for at least $\frac{n}{2}$ additional iterations with a probability of at least $1 - 2 n^{-1}$.
    
    Consequently, after the first~$\frac{n}{2}$ iterations, the optimum is sampled in each iteration with a probability of $(1 - \frac{1}{n})^n \geq 1/(2e)$. Thus, after $2 e \ln n$ additional iterations, the optimum is sampled with a probability of at least $1 - \big(1 - 1/(2e)\big)^{2 e \ln n} \geq 1 - n^{-1}$.
    
    Overall, by applying a union bound over all failure probabilities above, the UMDA needs at most $\frac{n}{2} + 2 e \ln n$ iterations to sample the optimum with a probability of at least $1 - 9 n^{-1}$.
\end{proof}

\section{Run Time Analysis of the \oea}\label{sec:lb}

In order to see how well the \oea compares to the UMDA on \dlb, we prove a precise run time in the order of $\Omega(n^3)$ (see \Cref{thm:oneOneEAonDLB}).
As our result for the UMDA only proves a run time bound that holds with high probability, for a fairer comparison, we also prove that the \oea needs $\Omega(n^3)$ with overwhelming probability when optimizing \dlb (see \Cref{thm:eaWHPBound}).

Our bound of order $\Omega(n^3)$ shows that the corresponding $O(n^3)$ bound from \cite{LehreN19UMDAonDLB} is tight apart from constant factors and lower order terms. We note that \cite{LehreN19UMDAonDLB} have also shown upper bounds for the run time of other evolutionary algorithms, again typically of order $O(n^3)$. We conjecture that for these setting an $\Omega(n^3)$ lower bound is valid as well, as our empirical results in \Cref{sec:experiments} suggest, but we do not prove this here.

To make the following result precise, we recall that the \oea for the maximization of $f\colon \{0,1\}^n \to \R$ is the simple EA which (i)~starts with an individual $x$ chosen uniformly at random from $\{0,1\}^n$ and then (ii)~in each iteration creates from the current solution $x$ an offspring $y$ via flipping each bit independently with probability $\frac 1n$ and replaces $x$ by $y$ if and only if $f(y) \ge f(x)$.

The following result determines precisely the run time of the \oea on \dlb. The proof, not surprisingly, takes some ideas from the precise analyses of the run time of the \oea on \LO by \citet{BottcherDN10} and \citet{Doerr19tcs}.

\begin{theorem}
    \label{thm:oneOneEAonDLB}
    In expectation, the \oea samples the optimum of \dlb after
    \[\frac 14 n^2 \left(n - \frac 12\right) \frac{(1+\frac 1 {n-1})^n - 1}{1 + \frac{1}{2 (n-1)}} = (1 + o(1)) \frac{e-1}{4} n^3\]
    fitness function evaluations.
\end{theorem}

\begin{proof}
  Let $\ell \in [0..n/2-1]$ and $b \in \{0,1\}$. Let $x \in \{0,1\}^n$ such that $\prefix(x) = \ell$ and $\db(x_{2\ell+1,2\ell+2}) = b$, in other words, such that $x_i = 1$ for all $i \le 2\ell$ and the contribution of the $\ell+1$-st block is $\db(x_{2\ell+1,2\ell+2}) = b$. Consider a run of the \oea on \dlb, started with the search point $x$ instead of a random initial solution. Let $X$ be the random variable describing the first time that a solution with $\prefix$-value greater than $\ell$ is found. 
  
  We first observe that $X$ is independent of $x_i$, $i > 2\ell+2$. In the case that $b = 0$, also~$X$ is independent of whether $(x_{2\ell+1,2\ell+2}) = (1,0)$ or $(x_{2\ell+1,2\ell+2}) = (0,1)$. Hence we use the notation $X_{\ell,b} := X$ without ambiguity later in this proof. 
  
  We compute $E[X]$. If $b = 1$, then $X$ follows a geometric law with success probability $p = (1-\frac 1n)^{2\ell} n^{-2}$; hence 
  \[E[X_{\ell,1}] = \frac 1p = \left(1 - \frac 1n\right)^{-2\ell} n^2.\] 
  If $b=0$, in principle we could also precisely describe the distribution of $X$, but since this is more complicated and we only regard expected run times in this proof, we avoid this and take the more simple route to only determine the expectation. From $x$, the \oea in one iteration finds a search point with $\prefix$-value greater than $\ell$ with probability $(1-\frac 1n)^{2\ell+1} n^{-1}$, since it hast to flip the only~$0$ in the first $2\ell + 2$ positions to a~$1$. With the same probability, it finds a search point with (unchanged) $\prefix$-value equal to $\ell$ and (higher) $b$-value $1$. Otherwise it finds no true improvement and stays with $x$ or an equally good search point. In summary, we have 
  \begin{align*}
  E[X_{\ell,0}] = 1 &+ \left(1-\frac 1n\right)^{2\ell+1} n^{-1} \cdot 0 + \left(1-\frac 1n\right)^{2\ell+1} n^{-1} E[X_{\ell,1}]\\ 
  &+ \left(1-2\left(1-\frac 1n\right)^{2\ell+1} n^{-1}\right) E[X_{\ell,0}],
  \end{align*}
  hence 
  \begin{align*}
  E[X_{\ell,0}] &=  \frac 12 \left(1-\frac 1n\right)^{-2\ell-1} n + \frac 12 E[X_{\ell,1}] .
  \end{align*}  
  
	Let $T_\ell$ denote the run time when starting with a random solution $x$ with $\prefix(x) \ge \ell$, that is, such that the first $2\ell$ bits of $x$ are all~$1$s and the remaining bits are random. Note that when during the subsequent optimization the prefix value increases, then this results is a random search point with prefix value $\ell+1$. From the above, we know that
	\begin{align*}
	E[T_\ell] &= \tfrac 14 (E[X_{\ell,1}] + E[T_{\ell+1}]) + \tfrac 12 (E[X_{\ell,0}] + E[T_{\ell+1}]) + \tfrac 14 (E[T_{\ell+1}])\\
	& = E[T_{\ell+1}] + \tfrac 14 E[X_{\ell,1}] + \tfrac 12 E[X_{\ell,0}].
	\end{align*}
	Noting that $E[T_{n/2}] = 0$, we can write $E[T_\ell] = \sum_{j = \ell}^{n/2 - 1} (E[T_j] - E[T_{j+1}])$. In particular, the run time $T \coloneqq T_0$ of the \oea with random initialization satisfies
  \begin{align*}
      E[T] & = \sum_{\ell=0}^{n/2-1} \left(\frac 12 E[X_{\ell,0}] + \frac 14 E[X_{\ell,1}]\right)
      = \sum_{\ell=0}^{n/2-1} \left(\frac{1}{4} \left(1-\frac 1n\right)^{-2\ell-1} n + \frac{1}{2} E[X_{\ell,1}]\right) \\
      & = \sum_{\ell = 0}^{n/2-1} \frac 14 \left(1-\frac 1n\right)^{-2\ell-1} n + \sum_{\ell = 0}^{n/2-1} \frac 12 \left(1 - \frac 1n\right)^{-2\ell} n^2 .
  \end{align*} 
  Noting that $\sum_{\ell = 0}^{n/2-1} \left(1 - \frac 1n\right)^{-2\ell} = \frac 12 (n-1) \frac{(1+\frac 1 {n-1})^n - 1}{1 + \frac{1}{2 (n-1)}}$, we obtain
  \[E[T] = \frac 14 n^2 \left(n - \frac 1{2}\right) \frac{(1+\frac 1 {n-1})^n - 1}{1 + \frac{1}{2 (n-1)}}.\qedhere\]
\end{proof}

The result above on the expected run time does not rule out that the typical run time of the \oea is much better, e.g., that it is quadratic with high probability, we add a short proof showing that with overwhelming probability, the run time of the \oea on \dlb is at least cubic.
We note that with deeper methods, namely martingale concentration inequalities allowing unbounded martingale differences, such as the ones used by~\cite{FanGL15} or~\cite{Kotzing16}, one could show more precise statements including that already lower-order deviations from the expectation almost surely do not happen.
However, we feel that such a result, while surely interesting, is not important enough to justify the effort.
Therefore, we only show the following weaker statement.

\begin{theorem}
    \label{thm:eaWHPBound}
  The probability that the \oea samples the optimum of \dlb in at most $\frac{n^3}{16e}$ iterations, is at most $\exp(-\frac{n}{64e})$.
\end{theorem}

\begin{proof}
  Consider a run of the \oea on \dlb. Let $(x_t)_{t \in \N}$ be the sequence of search points generated in this run, that is, $x_0$ is the random initial search point and $x_t$ is generated in iteration $t$ by mutating the current-best search point (breaking ties towards later generated ones). Let $t_i \coloneqq \min\{t \in \N \mid \prefix(x_t) \ge i\}$ be the first time a search point with prefix value at least $i$ is generated. Note that $T \coloneqq t_{n/2}$ is the run time of the algorithm (when we do not count the evaluation of the initial individual) and $t_0 = 0$. 
	
	Similar to what is well-known about the optimization of the \LO benchmark, we observe that if $t = t_i$, then $x_t$ is a search point with $2i$ leading~$1$s and the remaining bits independently and uniformly distributed in $\{0,1\}$. Consequently, $d_i \coloneqq t_{i+1} - t_i$ is independent of $t_0, \dots, t_i$. We note that with probability $1/4$, the search point $x_t$ has~$0$s in positions $2i+1$ and $2i+2$. In this case, the parent individual in the remaining run of the algorithm agrees with this $x_t$ in the first $2i+2$ bit positions until in iteration $t_{i+1}$ a better solution is found. Such a better solution is found with probability $p_i = (1-1/n)^{2i} n^{-2} \le n^{-2}$. With probability $(1-n^{-2})^{n^2-1} \ge 1/e$, this takes at least $n^2$ iterations. In summary, we see that, regardless of $d_0, \dots, d_{i-1}$, we have that $d_i$ is at least $n^2$ with probability at least $1/(4e)$. This conditional independence is enough to apply Chernoff-type concentration inequalities on the indicator random variables $X_i$ of the events $\{d_i \ge n^2\}$, see Lemma~11 in~\cite{DoerrJ10} or Section~1.10.2 in~\cite{Doerr20bookchapter}. Consequently, for $X = \sum_{i = 0}^{n/2 -1} X_i$ the multiplicative Chernoff bound (Theorem~\ref{thm:chernoff}) gives 
	\[
    	\Pr[X \le \tfrac n {16e}] = \Pr \big[X \le (1 - \tfrac 12) E[X]\big] \le \exp\big(- \tfrac 12 (\tfrac 12)^2 \tfrac{n}{8e} \big) = \exp(-\tfrac{n}{64e}).
	\]
	Since $X > n/(16e)$ implies that the run time $T = \sum_{i=0}^{n/2-1} d_i$ is larger than $n^3/(16e)$, we have shown our claim.
\end{proof}

\section{Experiments}
\label{sec:experiments}

In their paper, \cite{LehreN19UMDAonDLB} analyze the run time of many EAs on \dlb. For an optimal choice of parameters, they prove an expected run time of $O(n^3)$ for all considered algorithms.

Since these are only upper bounds and since we showed in \Cref{sec:lb} an $\Omega(n^3)$ lower bound only for the \oea, it is not clear how well the other algorithms actually perform against the UMDA, which has a run time in the order of $n^2 \ln n$ (\Cref{thm:UMDA_on_LO}) for appropriate parameters. Thus, we provide some empirical results in \Cref{fig:experiments} on how well these algorithms compare against each other.

\paragraph{Considered algorithms.} Besides the UMDA, we analyze the run time of the following EAs: the \oea, the $(\mu, \lambda)$~EA, and the $(\mu, \lambda)$~GA with uniform crossover, which are most of the EAs that \cite{LehreN19UMDAonDLB} consider in their paper. Lehre and Nguyen also analyze the $(1 + \lambda)$~EA and the $(\mu + 1)$~EA. However, in our preliminary experiments with $\mu = \lceil\ln n\rceil$ and $\lambda = \lceil\sqrt{n}\rceil$, they were always slower than the \oea, so we do not include these algorithms in \Cref{fig:experiments}.

Further, we also depict the run time of the \emph{mutual-information-maximizing input clustering} algorithm (MIMIC; \citealp{BonetIV96MIMIC}), which is one of the algorithms that \cite{LehreN19UMDAonDLB} also analyze empirically. The MIMIC is an EDA with a more sophisticated probabilistic model than the UMDA. This model is capable of capturing dependencies among the bit positions by storing a permutation~$\pi$ of indices and conditional probabilities. An individual is created bit by bit, following~$\pi$. Each bit is sampled with respect to a conditional probability, depending on the bit sampled in the prior position. The MIMIC updates its model as follow: similar to the UMDA, the MIMIC samples each iteration~$\lambda$ individuals and selects the~$\mu$ best. To update its richer probabilistic model, the MIMIC then searches for the position with the least (empirical) entropy, that is, the position that has the most~$0$s or~$1$s, and sets the new frequency equal to the number of $1$s in the selected subpopulation. Iteratively, for all remaining positions, it determines the position with the lowest conditional (empirical) entropy and sets the conditional frequencies to the conditional numbers of $1$s, where conditional is always with respect to the prior position in~$\pi$. Since a precise description of this algorithm would need a significant amount of space, we refer the reader to a recent paper by \cite{DoerrK20MIMIConEBOM} for more details.

\paragraph{Parameter choice.} For each of the EAs depicted in \Cref{fig:experiments}, we choose parameters such that the upper run time bound proven by \cite{LehreN19UMDAonDLB} is $O(n^3)$, that is, optimal. For both the $(\mu, \lambda)$~EA and the $(\mu, \lambda)$~GA, we choose $\mu = \lceil\ln n\rceil$ and $\lambda = 9 \mu$. For the latter, we further choose uniform crossover and, in each iteration, execute it with probability~$1/2$. Our parameter choices for the $(\mu, \lambda)$~EA stem from aiming for a constant ratio $\mu/\lambda$ as well as from the constraints $\lambda \geq c \ln n$ (for a sufficiently large constant~$c$) and $\lambda \geq (1 + \delta) e^2 \mu$ as stated by \cite{LehreN19UMDAonDLB}, choosing $\delta = 0.1$. For the $(\mu, \lambda)$~GA, we choose the same values for~$\mu$ and~$\lambda$, as the algorithm uses the same selection mechanism as the $(\mu, \lambda)$~EA and as the values comply with the constraints stated by \cite{LehreN19UMDAonDLB}. The crossover probability of~$1/2$ is chosen as it is a constant probability that statisfies the constraint to be at most $1 - (1 + \delta)e\mu/\lambda$, where we choose $\delta = 0.1$ as before.

For the UMDA and the MIMIC, we choose $\mu = \lceil3 n \ln n\rceil$ and $\lambda = 12 \mu$. Our parameter choice for the UMDA is based on a constant ratio $\lambda/\mu$ and on a preliminary search for the first integer factors for~$\mu$ and~$\lambda$ such that the UMDA is successful in a reasonable time. For the MIMIC, we speculate that the similarities with the UMDA imply that these parameter values are suitable as well. In any case, with these parameter values all runs of these two algorithms were successful, that is, the optimum of \dlb was found within $10n^3$ function evaluations.

\paragraph{Results.} When interpreting the plots from \Cref{fig:experiments}, please note that both axes use a logarithmic scale. Thus, any polynomial is depicted as a linear function, where the slope depends on the exponent of the highest term. \Cref{fig:experiments} clearly shows two separate regimes: The EAs with a larger slope and the EDAs with a smaller slope.

We see that the \oea from $n \ge 200$ on performs the worst out of all algorithms. The EAs all have a very similar slope, which ranges, according to a regression to fit a power law, from $\Theta(n^{2.9967})$ to $\Theta(n^{3.1})$. This indicated that all of these EAs have a run time of $\Theta(n^3)$, as we proved for the \oea in \Cref{thm:oneOneEAonDLB}. In contrast, the UMDA and the MIMIC have a smaller slope, about $\Theta(n^{2.2957})$ and $\Theta(n^{2.3096})$, respectively, suggesting that our upper bound of $O(n^2 \log n)$ for the UMDA (\Cref{thm:UMDA_on_LO}) is close to its true run time. Further note that the variance of the EDAs is very small, suggesting that their run time holds with a higher concentration bound than what we prove for the UMDA in \Cref{thm:UMDA_on_LO}.

Interestingly, the MIMIC behaves very similarly to the UMDA. However, this may be a result of the same parameter choice of both algorithms. \cite{LehreN19UMDAonDLB} also analyze the MIMIC empirically, with choices of~$\mu$ in the orders of~$\sqrt{n}$, $\sqrt{n}\log n$, and~$n$. In their results, the run time of the MIMIC is slightly faster; for $n = 100$, the median of the empirical run time of the MIMIC as reported by \cite{LehreN19UMDAonDLB} is in the order of $2^{19} \approx 5.2 \cdot 10^5$. In \Cref{fig:experiments}, for the same value of~$n$, the run time of the MIMIC is about $6 \cdot 10^5$, which is very close. This suggests that the parameter regime without strong genetic drift, which we consider in this paper for the UMDA, is not that important for the MIMIC. While the MIMIC seems to have a larger tolerable parameter regime, it is impressive that the UMDA has a similar performance to the MIMIC.\footnote{It is important to note that we only report the number of fitness evaluations in \Cref{fig:experiments}. When comparing the time to perform an update, the UMDA is faster, as the model update can be computed in time $\Theta(n)$, whereas the one for the MIMIC takes time $\Theta(n^2)$ due to the iterated search of the minimal conditional entropy.}

\begin{figure}
    \centering
    \includegraphics[width = 12 cm]{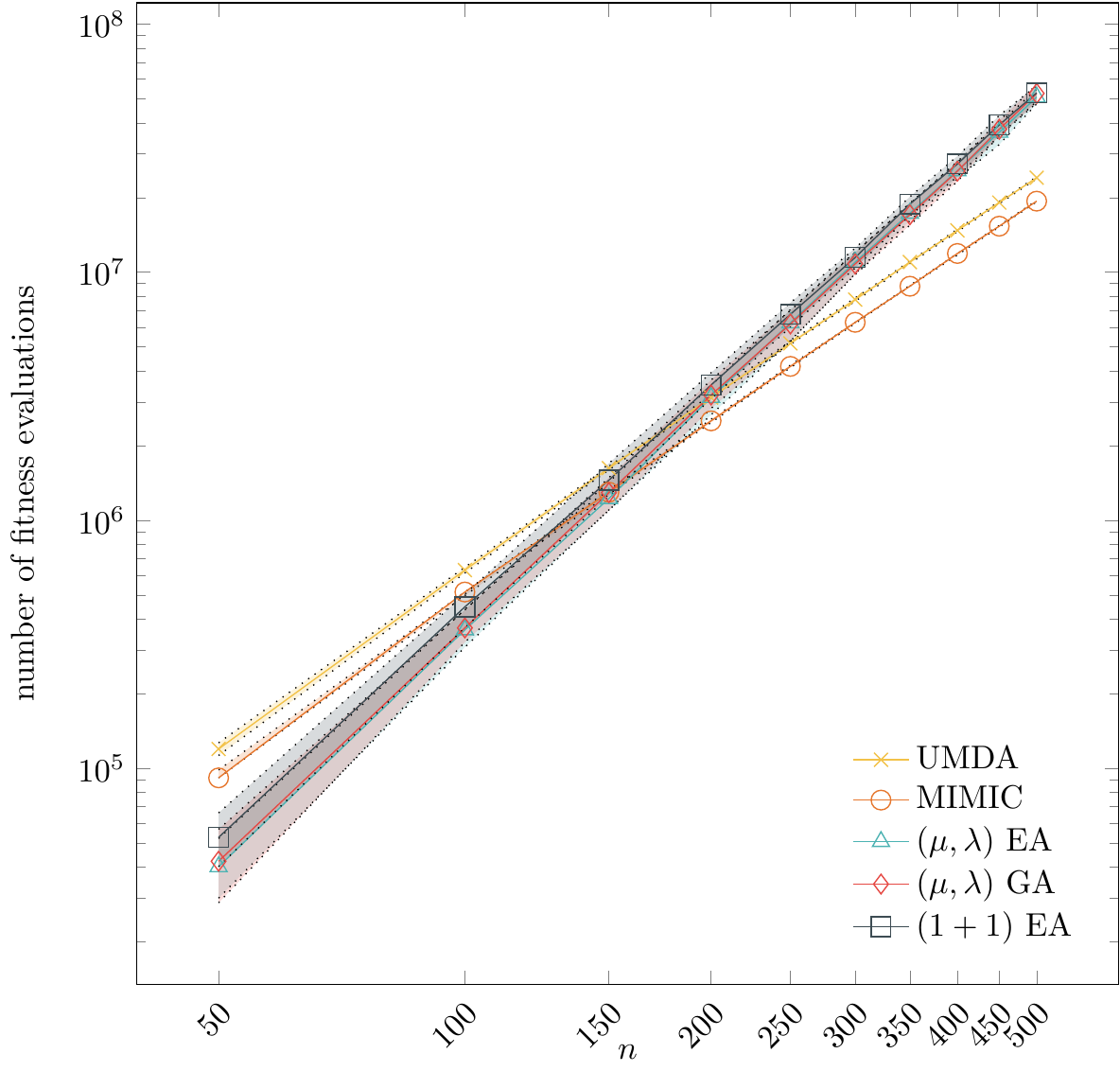}
    \caption{A log-log plot depicting the number of fitness evaluations of various algorithms until the optimum of \dlb is sampled for the first time against the input size~$n$ from~$50$ to~$300$ in steps of~$50$. For each value of~$n$, $100$ independent runs were started per algorithm. The lines depict the median of the~$50$ runs of an algorithm, and the shaded areas denote the center $50$\,\%.\newline
        Please refer to \Cref{sec:experiments} for more details.
    }
    \label{fig:experiments}
\end{figure}

\subsection{Different Regimes of the UMDA's Population Size}
\label{sec:variableMu}

\Cref{fig:experiments} compares the run time of the UMDA to that of other algorithms. To this end, we chose the parameters in the regime without strong genetic drift, as proposed in \Cref{thm:UMDA_on_LO}. However, \cite{LehreN19UMDAonDLB} prove a lower bound of $\exp\big(\Omega(\mu)\big)$ for the UMDA on \dlb, which can be lower than our upper bound of $O(n^2 \log n)$ if~$\mu$ is sufficiently small. Thus, in \Cref{fig:variableMu}, we analyze the impact of~$\mu$ on the run time of the UMDA on \dlb for $n = 300$.

\paragraph{Parameter choice.} In each run, we choose $\lambda = 12 \mu$, as we did before; $\mu$ ranges from $2^1$ to $2^{12}$ in powers of~$2$. Note that our range for~$\mu$ ends with $2^{12} = 4{,}096$, which is less than our choice of $\mu = \lceil3 n \ln n\rceil$ in the previous experiment, which results in a value of $\mu = 5{,}134$.

\paragraph{Results.} In \Cref{fig:variableMu}, we show the run times for those values of $\mu$ where at least one out of 100 runs was successful, that is, where the UMDA found the optimum of \dlb within $10n^3$ function evaluations. There are two very distinct ranges of $\mu$ leading to successful runs. The first regime consists of the values~$2$, $4$, and~$8$; note that the effect of genetic drift is strong in this regime. For each value of~$\mu$, the number of fitness evaluations is at least $10^7$, which is worse than the performance of the UMDA in \Cref{fig:experiments}. Interestingly, in this regime the UMDA is successful in every run. This suggests that the range of frequency values for the UMDA is coarse enough such that it is able to quickly change frequencies from~$\frac{1}{n}$ to $1 - \frac{1}{n}$. If a frequency is at~$\frac{1}{n}$, it takes some time until a~$1$ is sampled and selected, but if this occurs, the respective frequency has a decent chance of being increased to $1 - \frac{1}{n}$.

For values of~$\mu$ between $2^4$ to $2^9$, the range of frequency values increases, consequently leading to longer times for a frequency to be increased from~$\frac{1}{n}$ to $1 - \frac{1}{n}$. The growth of the plot from the regime of smaller values of~$\mu$ suggests a superpolynomial run time in~$\mu$ for this behavior. Thus, it is not surprising that no run of the UMDA is successful for the medium values of~$\mu$.

The second successful regime of~$\mu$ consists of the values~$2^{10}$, $2^{11}$, and~$2^{12}$. While not all runs are successful, in particular not those for the smallest value of $\mu$, the trend of the curve seems to be linear, which suggests a polynomial run time of the UMDA (because of the log-log scale). Further, the run time is comparable to that of the UMDA seen in \Cref{fig:experiments}. Likely, starting with $\mu = 2^{10}$, the UMDA enters the regime where the effect of genetic drift starts to diminish. For $\mu = 2^{10}$, the effect is still rather large, and oftentimes frequencies reach the lower border~$\frac{1}{n}$, however, not all of the time. With increasing~$\mu$, this happens less and less. For $\mu = 2^{12}$, almost all runs are successful. In comparison, in \Cref{fig:experiments}, \emph{all} runs of the UMDA were successful, which uses a slightly larger value of~$\mu$ than~$2^{12}$.

Overall, the results from \Cref{fig:variableMu} show an interesting transition from the run time of the UMDA on \dlb when the effect of genetic drift vanishes.

\begin{figure}
    \centering
    \includegraphics[width = 12 cm]{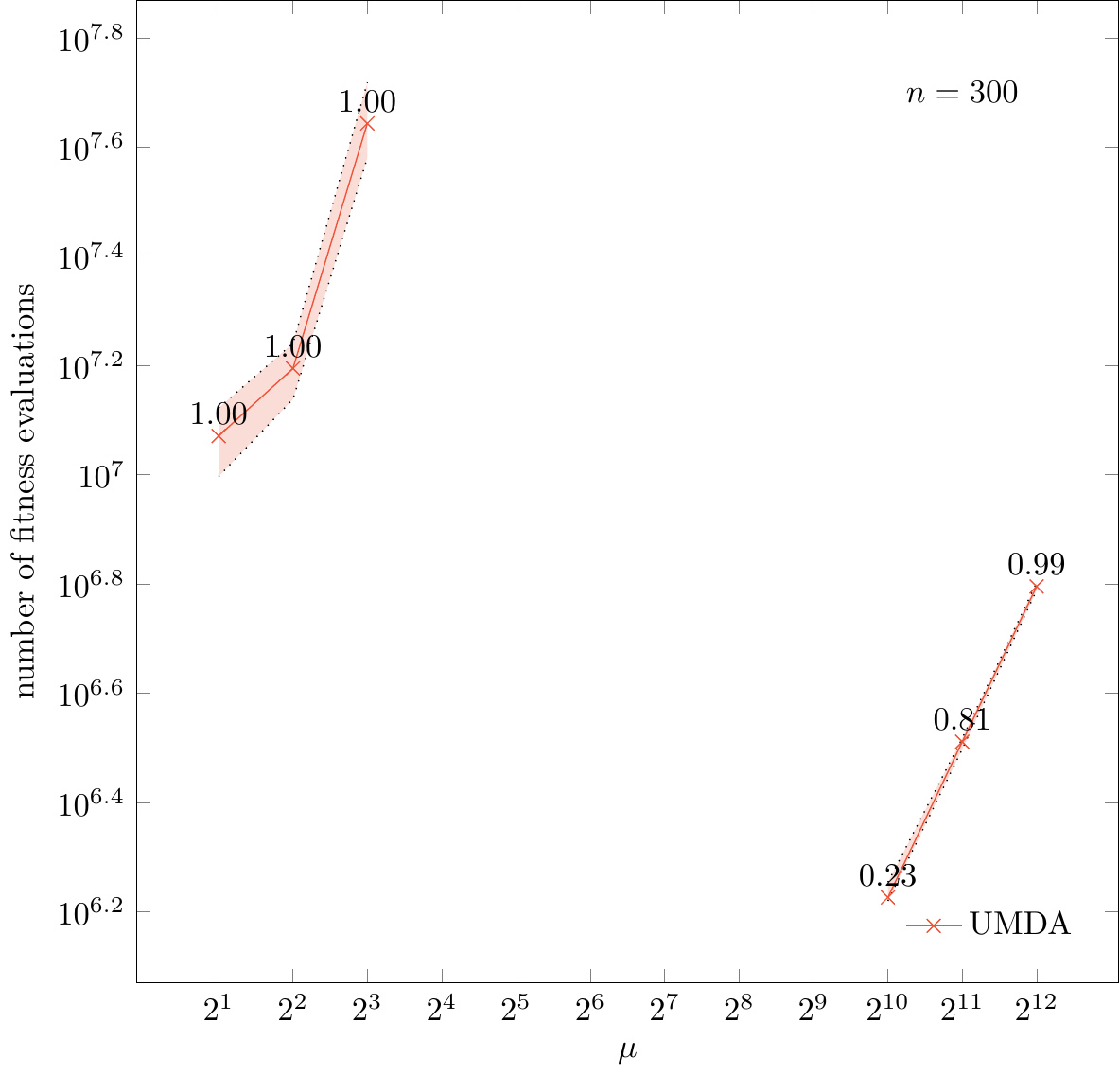}
    \caption{A log-log plot depicting the number of fitness evaluations of the UMDA until the optimum of \dlb for $n = 300$ is sampled for the first time against the parameter~$\mu$ from~$2^1$ to~$2^{12}$, doubling the value in each step. For each value of~$n$, $100$ independent runs were started. We terminated each run after $10 n^3 = 2.7 \cdot 10^{8}$ function evaluations if the UMDA did not find the optimum until then.
        \newline
        The number above each data point denotes the ratio of runs in which the UMDA actually found an optimum. The lines depict the median of the \emph{successful} runs, and the shaded areas denote their center $50$\,\%.
        \newline
        Note that for the values of~$\mu$ from $2^4$ to $2^9$, no run was successful. Thus, no data points are depicted.
        \newline
        Please refer to \Cref{sec:variableMu} for more details.
    }
    \label{fig:variableMu}
\end{figure}

\section{Conclusion}
\label{sec:conclusion}

We conducted a rigorous run time analysis of the UMDA on the \DLB function. In particular, it shows that the algorithm with the right parameter choice finds the optimum in $O(n^2 \log n)$ fitness evaluations with high probability (\Cref{thm:UMDA_on_LO}). This result shows that the lower bound by \cite{LehreN19UMDAonDLB}, which is exponential in~$\mu$, is not due to the UMDA being ill-suited for coping with epistasis and deception, but rather due to an unfortunate choice of the algorithm's parameters. For several EAs, \cite{LehreN19UMDAonDLB} showed a run time bound of $O(n^3)$ on \dlb. We proved a matching lower bound for the \oea (\Cref{sec:lb}) and conducted experiments which suggest that also other EAs perform worse than the UMDA on \dlb (\Cref{sec:experiments}). In this light, our result suggests that the UMDA can handle epistasis and deception even better than many evolutionary algorithms and that the UMDA does so similar to a more complex EDA.

Our run time analysis holds for parameter regimes that prevent genetic drift. When comparing our run time with the one shown by~\cite{LehreN19UMDAonDLB}, we obtain a strong suggestion for running EDAs in regimes of low genetic drift. In contrast to the work of \cite{LenglerSW18cGAmediumStepSizes} that indicates moderate performance losses due to genetic drift, here we obtain the first fully rigorous proof of such a performance loss, and in addition one that is close to exponential in $n$ (the $\exp(\Omega(\mu))$ lower bound of~\cite{LehreN19UMDAonDLB} holds for $\mu$ up to $o(n)$). Our proven upper and lower bound also show that the UMDA has an advantage in coping with local optima compared to EAs. Such an observation has previously only been made for the compact genetic algorithm (when optimizing jump functions, see \cite{HasenoehrlS18cGAonJump,Doerr19cGAonJump}).

On the technical side, our result indicates that the regime of low genetic drift admits relatively simple and natural analyses of run times of EDAs, in contrast, e.g., to the level-based methods previously used in comparable analyses, e.g., by~\cite{DangL15UMDAonLO} and \cite{LehreN19UMDAonDLB}.

We conjecture that our result can be generalized to a version of the \dlb function with a block size of $k \leq n$.

\subsection*{Acknowledgments}

We thank the anonymous reviewers of this paper, who provided valuable feedback that improved the paper in various aspects.

This work was supported by COST Action CA15140 and by  a public grant as part of the Investissements d'avenir project, reference ANR-11-LABX-0056-LMH, LabEx LMH, in a joint call with Gaspard Monge Program for optimization, operations research and their interactions with data sciences.

\bibliographystyle{apalike}
\bibliography{references}

\begin{thebibliography}{}

\bibitem[B{\"o}ttcher et~al., 2010]{BottcherDN10}
B{\"o}ttcher, S., Doerr, B., and Neumann, F. (2010).
\newblock Optimal fixed and adaptive mutation rates for the {L}eading{O}nes
  problem.
\newblock In {\em Proc. of PPSN '10}, pages 1--10. Springer.

\bibitem[Chen et~al., 2009]{ChenLTY09UMDAonSubstring}
Chen, T., Lehre, P.~K., Tang, K., and Yao, X. (2009).
\newblock When is an estimation of distribution algorithm better than an
  evolutionary algorithm?
\newblock In {\em Proc.~of CEC~'09}, pages 1470--1477.

\bibitem[Corus et~al., 2018]{CorusDEL18}
Corus, D., Dang, D., Eremeev, A.~V., and Lehre, P.~K. (2018).
\newblock Level-based analysis of genetic algorithms and other search
  processes.
\newblock {\em {IEEE} Transactions on Evolutionary Computation},
  22(5):707--719.

\bibitem[Dang et~al., 2016]{DangFKKLOSS16EAonJumpDiversityMechanisms}
Dang, D., Friedrich, T., Kötzing, T., Krejca, M.~S., Lehre, P.~K., Oliveto,
  P.~S., Sudholt, D., and Sutton, A.~M. (2016).
\newblock Escaping local optima with diversity mechanisms and crossover.
\newblock In {\em Proc.~of GECCO~'16}, pages 645--652.

\bibitem[Dang et~al., 2018]{DangFKKLOSS18EAonJumpWithCrossover}
Dang, D., Friedrich, T., Kötzing, T., Krejca, M.~S., Lehre, P.~K., Oliveto,
  P.~S., Sudholt, D., and Sutton, A.~M. (2018).
\newblock Escaping local optima using crossover with emergent diversity.
\newblock {\em {IEEE Transactions on Evolutionary Computation}},
  22(3):484--497.

\bibitem[Dang and Lehre, 2015]{DangL15UMDAonLO}
Dang, D. and Lehre, P.~K. (2015).
\newblock Simplified runtime analysis of estimation of distribution algorithms.
\newblock In {\em Proc.~of {GECCO}~'15}, pages 513--518.

\bibitem[Dang and Lehre, 2016]{DangL16algo}
Dang, D. and Lehre, P.~K. (2016).
\newblock Runtime analysis of non-elitist populations: from classical
  optimisation to partial information.
\newblock {\em Algorithmica}, 75(3):428--461.

\bibitem[{De Bonet} et~al., 1996]{BonetIV96MIMIC}
{De Bonet}, J.~S., {Isbell, Jr.}, C.~L., and Viola, P.~A. (1996).
\newblock {MIMIC:} finding optima by estimating probability densities.
\newblock In {\em Proc.~of NIPS~'96}, pages 424--430.

\bibitem[Doerr, 2019a]{Doerr19tcs}
Doerr, B. (2019a).
\newblock Analyzing randomized search heuristics via stochastic domination.
\newblock {\em Theoretical Computer Science}, 773:115--137.

\bibitem[Doerr, 2019b]{Doerr19cGAonJump}
Doerr, B. (2019b).
\newblock A tight runtime analysis for the {cGA} on jump functions: {EDA}s can
  cross fitness valleys at no extra cost.
\newblock In {\em Proc.~of {GECCO}~'19}, pages 1488--1496.

\bibitem[Doerr, 2020a]{Doerr20gecco}
Doerr, B. (2020a).
\newblock Does comma selection help to cope with local optima?
\newblock In {\em Proc. of GECCO '20}, pages 1304--1313.

\bibitem[Doerr, 2020b]{Doerr20bookchapter}
Doerr, B. (2020b).
\newblock Probabilistic tools for the analysis of randomized optimization
  heuristics.
\newblock In {\em Theory of Evolutionary Computation: Recent Developments in
  Discrete Optimization}, pages 1--87. Springer.
\newblock Also available at \url{https://arxiv.org/abs/1801.06733}.

\bibitem[Doerr and Johannsen, 2010]{DoerrJ10}
Doerr, B. and Johannsen, D. (2010).
\newblock Edge-based representation beats vertex-based representation in
  shortest path problems.
\newblock In {\em Genetic and Evolutionary Computation Conference, GECCO 2010},
  pages 759--766. ACM.

\bibitem[Doerr and Künnemann, 2015]{DoerrK15LambdaEAonLinearFunctions}
Doerr, B. and Künnemann, M. (2015).
\newblock Optimizing linear functions with the (1+{\(\lambda\)}) evolutionary
  algorithm -- {Different} asymptotic runtimes for different instances.
\newblock {\em Theoretical Computer Science}, 561:3--23.

\bibitem[Doerr and K{\"{o}}tzing, 2019]{DoerrK19}
Doerr, B. and K{\"{o}}tzing, T. (2019).
\newblock Multiplicative up-drift.
\newblock In {\em Proc. of GECCO '19}, pages 1470--1478.

\bibitem[Doerr and Krejca, 2018]{DoerrK18sigcGA}
Doerr, B. and Krejca, M.~S. (2018).
\newblock Significance-based estimation-of-distribution algorithms.
\newblock In {\em Proc.~of GECCO~'18}, pages 1483--1490.

\bibitem[Doerr and Krejca, 2020a]{DoerrK20MIMIConEBOM}
Doerr, B. and Krejca, M.~S. (2020a).
\newblock Bivariate estimation-of-distribution algorithms can find an
  exponential number of optima.
\newblock In {\em Proc. of GECCO '20}, pages 796--804.

\bibitem[Doerr and Krejca, 2020b]{DoerrK20UMDAonDLB}
Doerr, B. and Krejca, M.~S. (2020b).
\newblock The univariate marginal distribution algorithm copes well with
  deception and epistasis.
\newblock In {\em Proc. of EvoCOP '20}, pages 51--66.

\bibitem[Doerr and Zheng, 2020]{DoerrZ20tec}
Doerr, B. and Zheng, W. (2020).
\newblock Sharp bounds for genetic drift in estimation-of-distribution
  algorithms.
\newblock {\em IEEE Transactions on Evolutionary Computation}.
\newblock To appear.

\bibitem[Droste, 2006]{Droste06cGAonLinearFunctions}
Droste, S. (2006).
\newblock A rigorous analysis of the compact genetic algorithm for linear
  functions.
\newblock {\em Natural Computing}, 5(3):257--283.

\bibitem[Droste et~al., 2002]{DrosteJW02OnePlusOneEAonJump}
Droste, S., Jansen, T., and Wegener, I. (2002).
\newblock On the analysis of the {(1+1)} evolutionary algorithm.
\newblock {\em Theoretical Computer Science}, 276(1-2):51--81.

\bibitem[Fan et~al., 2015]{FanGL15}
Fan, X., Grama, I., and Liu, Q. (2015).
\newblock Exponential inequalities for martingales with applications.
\newblock {\em Electronic Journal of Probability}, 20:1--22.

\bibitem[Hasenöhrl and Sutton, 2018]{HasenoehrlS18cGAonJump}
Hasenöhrl, V. and Sutton, A.~M. (2018).
\newblock On the runtime dynamics of the compact genetic algorithm on jump
  functions.
\newblock In {\em Proc.~of GECCO~'18}, pages 967--974.

\bibitem[Hoeffding, 1963]{Hoeffding63ChernoffBound}
Hoeffding, W. (1963).
\newblock Probability inequalities for sums of bounded random variables.
\newblock {\em Journal of the American Statistical Association},
  58(301):13--30.

\bibitem[K{\"{o}}tzing, 2016]{Kotzing16}
K{\"{o}}tzing, T. (2016).
\newblock Concentration of first hitting times under additive drift.
\newblock {\em Algorithmica}, 75:490--506.

\bibitem[Krejca and Witt, 2020a]{KrejcaW20UMDAlowerBoundOneMax}
Krejca, M.~S. and Witt, C. (2020a).
\newblock Lower bounds on the run time of the univariate marginal distribution
  algorithm on {OneMax}.
\newblock {\em Theoretical Computer Science}, 832:143--165.

\bibitem[Krejca and Witt, 2020b]{KrejcaW18EDAoverview}
Krejca, M.~S. and Witt, C. (2020b).
\newblock Theory of estimation-of-distribution algorithms.
\newblock In {\em Theory of Evolutionary Computation: Recent Developments in
  Discrete Optimization}, pages 405--442. Springer.
\newblock Also available at \url{http://arxiv.org/abs/1806.05392}.

\bibitem[Lehre, 2011]{Lehre11}
Lehre, P.~K. (2011).
\newblock Fitness-levels for non-elitist populations.
\newblock In {\em Proc. of {GECCO} '11}, pages 2075--2082.

\bibitem[Lehre and Nguyen, 2017]{LehreN17UMDAonOneMax}
Lehre, P.~K. and Nguyen, P. T.~H. (2017).
\newblock Improved runtime bounds for the univariate marginal distribution
  algorithm via anti-concentration.
\newblock In {\em Proc.~of {GECCO}~'17}, pages 1383--1390.

\bibitem[Lehre and Nguyen, 2019]{LehreN19UMDAonDLB}
Lehre, P.~K. and Nguyen, P. T.~H. (2019).
\newblock On the limitations of the univariate marginal distribution algorithm
  to deception and where bivariate {EDA}s might help.
\newblock In {\em Proc.~of {FOGA}~'19}, pages 154--168.

\bibitem[Lengler et~al., 2018]{LenglerSW18cGAmediumStepSizes}
Lengler, J., Sudholt, D., and Witt, C. (2018).
\newblock Medium step sizes are harmful for the compact genetic algorithm.
\newblock In {\em Proc.~of GECCO~'18}, pages 1499--1506.

\bibitem[Mühlenbein and Paaß, 1996]{MuehlenbeinP96UMDA}
Mühlenbein, H. and Paaß, G. (1996).
\newblock From recombination of genes to the estimation of distributions {I.}
  {B}inary parameters.
\newblock In {\em Proc.~of {PPSN}~'96}, pages 178--187.

\bibitem[Pelikan et~al., 2015]{PelikanHL15SurveyOnEDAs}
Pelikan, M., Hauschild, M., and Lobo, F.~G. (2015).
\newblock Estimation of distribution algorithms.
\newblock In {\em Springer Handbook of Computational Intelligence}, pages
  899--928. Springer.

\bibitem[Sudholt and Witt, 2019]{SudholtW19cGAandACOonOneMax}
Sudholt, D. and Witt, C. (2019).
\newblock On the choice of the update strength in estimation-of-distribution
  algorithms and ant colony optimization.
\newblock {\em Algorithmica}, 81(4):1450--1489.

\bibitem[Witt, 2018]{Witt18DominoConvergence}
Witt, C. (2018).
\newblock Domino convergence: why one should hill-climb on linear functions.
\newblock In {\em Proc.~of GECCO~'18}, pages 1539--1546.

\bibitem[Witt, 2019]{Witt19}
Witt, C. (2019).
\newblock Upper bounds on the running time of the univariate marginal
  distribution algorithm on {OneMax}.
\newblock {\em Algorithmica}, 81(2):632--667.

\end{thebibliography}

\end{document}